\documentclass{article}

\PassOptionsToPackage{numbers, compress}{natbib}

\usepackage[preprint]{neurips_2025}

\usepackage{mathtools}

 \usepackage{relsize}

\usepackage{etoolbox}
\usepackage{xparse}

\newcommand{\R}{\Rbb}

\newcommand{\Prob}{\mathbb{P}}

\renewcommand{\vec}[1]{\ensuremath{\mathbf{#1}}}
\newcommand{\vecs}[1]{\ensuremath{\mathbf{\boldsymbol{#1}}}}
\newcommand{\mat}[1]{\ensuremath{\mathbf{#1}}}
\newcommand{\mats}[1]{\ensuremath{\mathbf{\boldsymbol{#1}}}}
\newcommand{\ten}[1]{\mat{\ensuremath{\boldsymbol{\mathcal{#1}}}}}

\usepackage{pgffor}

\foreach \x in {A,...,Z}{%
\expandafter\xdef\csname \x bb\endcsname{\noexpand\ensuremath{\noexpand\mathbb{\x}}}
}

\foreach \x in {A,...,Z}{%
\expandafter\xdef\csname \x cal\endcsname{\noexpand\ensuremath{\noexpand\mathcal{\x}}}
}

\foreach \x in {A,...,Z}{%
\expandafter\xdef\csname \x t\endcsname{\noexpand\ensuremath{\noexpand\ten{\x}}}
}

\foreach \x in {A,...,Z}{%
\expandafter\xdef\csname \x b\endcsname{\noexpand\ensuremath{\noexpand\mat{\x}}}
}

\foreach \x in {a,...,r}{%
\expandafter\xdef\csname \x b\endcsname{\noexpand\ensuremath{\noexpand\vec{\x}}}
}
\foreach \x in {t,...,z}{%
\expandafter\xdef\csname \x b\endcsname{\noexpand\ensuremath{\noexpand\vec{\x}}}
}

\newcommand{\x}{\vec{x}}

\newcommand{\nstates}{n}

\newcommand{\szerosymbol}{\alpha}
\newcommand{\szero}{\vecs{\szerosymbol}}

\newcommand{\sinfsymbol}{\omega}
\newcommand{\sinf}{\vecs{\sinfsymbol}}

\DeclareDocumentCommand{\wa}{  O{A} O{\szero} O{\sinf} }%
{(#2,\{\mat{#1}^\sigma\}_{\sigma\in\Sigma},#3)}
\DeclareDocumentCommand{\waR}{  O{A} O{\Rbb^\nstates} O{\szero} O{\sinf} }%
{(#2,#3,\{\mat{#1}^\sigma\}_{\sigma\in\Sigma},#4)}

\newcommand{\vvsinfsymbol}{\Omega}
\newcommand{\vvsinf}{\mats{\vvsinfsymbol}}
\DeclareDocumentCommand{\vvwa}{  O{A} O{\szero} O{\vvsinf} }%
{(#2,\{\mat{#1}^\sigma\}_{\sigma\in\Sigma},#3)}

\newcommand{\tzerosymbol}{\alpha}
\newcommand{\tzero}{\vecs{\tzerosymbol}}

\newcommand{\tinfsymbol}{\omega}
\newcommand{\tinf}{\vecs{\tinfsymbol}}

\DeclareDocumentCommand{\wta}{ O{T} O{\Rbb^\nstates} O{\tzero} O{\tinf} O{\Fcal}}%
{(#2,#3,\{\ten{#1}^g\}_{g\in #5_{\geq 1}},\{#4^\sigma\}_{\sigma\in #5_0})}

\DeclareDocumentCommand{\trees}{g}{\IfNoValueTF{#1}{\mathfrak{T}}{\mathfrak{T}_{#1}}}
\DeclareDocumentCommand{\contexts}{g}{\IfNoValueTF{#1}{\mathfrak{C}}{\mathfrak{C}_{#1}}}

\newcommand{\gwmprod}{\diamond}

\DeclareDocumentCommand{\gwm}{  O{M} O{\Fbb^\nstates}}{(#2, \{\ten{#1}^x\}_{x\in\Sigma})}
\DeclareDocumentCommand{\gwmcirc}{  O{M} O{\Rbb^\nstates}}{(#2, \{\mat{#1}^\sigma\}_{\sigma\in\Sigma})}
\DeclareDocumentCommand{\dgwm}{ O{M} O{\Fbb^\nstates}}{(#2, \{\ten{#1}^x\}_{x\in\Sigma},\gwmprod)}

\usepackage{pgffor}
\foreach \x in {A,...,Z}{%
\expandafter\xdef\csname \x bb\endcsname{\noexpand\ensuremath{\noexpand\mathbb{\x}}}
}

\foreach \x in {A,...,Z}{%
\expandafter\xdef\csname \x cal\endcsname{\noexpand\ensuremath{\noexpand\mathcal{\x}}}
}

\usepackage{pgffor}
\foreach \x in {A,...,Z}{%
\expandafter\xdef\csname \x ten\endcsname{\noexpand\ensuremath{\noexpand\ten{\x}}}
}

\foreach \x in {A,...,Z}{%
\expandafter\xdef\csname \x mat\endcsname{\noexpand\ensuremath{\noexpand\mat{\x}}}
}

\foreach \x in {a,...,z}{%
\expandafter\xdef\csname \x vec\endcsname{\noexpand\ensuremath{\noexpand\mat{\x}}}
}

\usepackage[symbol]{footmisc}

\usepackage{subfigure}
\usepackage{hyperref}
\usepackage{url}
\usepackage{graphicx}
\usepackage{amssymb}
\usepackage{amsmath}
\usepackage{mathtools}
\usepackage{amsthm}
\usepackage{booktabs}
\usepackage{enumitem}
\usepackage{tikz}
\usepackage{subcaption}
\usepackage{bbm}
\usepackage{algorithm}
\usepackage{algpseudocode}
\usepackage{tabularx}
\usepackage{wrapfig}

\usetikzlibrary{calc, positioning}

\theoremstyle{definition}
\newtheorem{definition}{Definition}[section]

\newtheorem{lemma}{Lemma}[section]
\newtheorem{theorem}{Theorem}[section]
\newtheorem{corollary}{Corollary}[section]
\newtheorem{fact}{Fact}[section]
\newtheorem{example}{Example}

\newcommand{\lightgreycomment}[1]{%
    \textcolor{gray}{\Comment{#1}}%
}

\title{From Expressivity to Sample Complexity: Narrow Teachers for Transformers via C-RASP}

\def\mystrut{\rule{0pt}{1.1\normalbaselineskip}}
\author{
\begin{tabular}{@{}l}
Michael Rizvi-Martel$^{1}$\thanks{Corresponding author. Contact: \href{mailto:michael.rizvi-martel@mila.quebec}{\texttt{\footnotesize michael.rizvi-martel@mila.quebec}}}\quad Satwik Bhattamishra$^{2}$\quad \mystrut Guillaume Rabusseau$^{1}$ \quad
Michael Hahn$^{4}$\mystrut \\
\end{tabular}\\ [1.4em]
$^1$Mila \& Universit{\'e} de Montr{\'e}al \quad $^2$University of Oxford\quad
$^4$Saarland University\\
}

\begin{document}

\maketitle

\begin{abstract}
    A theoretical understanding of Transformers is crucial to better understand the capacities and limitations of large language models (LLMs). There is much work analyzing the expressivity of attention-based models. By proposing handcrafted weights or using computational complexity arguments, a large amount of past theoretical works have sought to characterize which tasks are and which are not in the hypothesis class of Transformer models.
    However, little work investigates the \textit{learnability} of such solutions. In this work, we make progress towards this goal. Inspired by recent loss landscape analysis work, we propose preliminary sample complexity bounds for learning C-RASP constructions with Transformers.
\end{abstract}

\section{Introduction}
Theoretical understanding of Transformer models is crucial to better understanding the capacities and limitations of modern large language models (LLMs). Current theoretical analyses of transformers focus primarily on expressivity, characterizing what these models can or cannot encode in their weights, and situating them in known complexity classes~\citep{bhattamishra2024separations,hahn2020theoretical,rizvi2024simulating,merrill2022saturated}. Notably, RASP~\citep{weiss2021thinking} and C-RASP~\citep{yang2024counting,yang2025knee} established an equivalence between transformers and handcrafted programming languages with specific features.
However, much of this work is largely detached from learnability, leaving open the question of how such expressive capabilities can actually be acquired in training.
Recent work by~\cite{chiang2022loss} suggests that generalization in deep learning can be largely attributed to the volume of "good" solutions in the loss landscape. This idea is formalized by~\cite{buzaglo2024uniform}, who show that, within a PAC learning framework, a randomly initialized network generalizes well if there exists a ``narrow teacher'' network consistent with the labels. They derive a sample complexity bound based on the probability of sampling such a narrow teacher. Building on this framework, we demonstrate how C-RASP Transformers, which define a large class of constructions, can yield sample complexity bounds. These results shed light on how constant-size constructions such as Dyck-1 and $a^nb^n$ are easily learned by Transformers.

Our main contributions are as follows: (i) We extend the framework of Buzaglo et al. to Transformers and provide a constraint-counting argument showing how wide student networks can implement narrow teachers. (ii) To the best of our knowledge, our work is among the first to bridge expressivity and learnability by transforming a class of expressivity results into sample complexity results.
\paragraph{Background: A Brief Introduction to C-RASP}
RASP is a programming language introduced by~\cite{weiss2021thinking} characterizing the expressivity of Transformers. C-RASP~\citep{yang2024counting} is a variant allowing compilation in future-masked soft-attention Transformers with no restrictions on the input length.
C-RASP operations are either boolean (e.g., comparison or logical AND) or pertain to counting (e.g., counting the number of tokens up to position $i$).
Algorithm~\ref{alg:dyck1} gives an example program for recognizing Dyck-1~\citep{yang2024counting, bhattamishra-etal-2020-ability}, the language of well-balanced parentheses with 1 parenthesis type.

\section{Theoretical Contribution}
\paragraph{Problem Setup}We assume that $N$ sequences of length at most $T$ are sampled i.i.d. from some distribution $\mathcal{D}$ s.t. $\mathcal{S} := \{\mathbf{x}_n\}_{n=1}^N \sim \mathcal{D}^N$, where $\mathbf{x}_n \in \mathcal{V}^{\leq T}$ and  $\mathcal{V} = \{1,\hdots, |\mathcal{V}|\}$ is a vocabulary of discrete tokens. We also assume the existence of a teacher model $h^*:\mathcal{V}^{\leq T} \to \{\pm 1\}$.
The true risk of a predictor $h: \mathcal{V}^{\leq T} \to \{\pm1\}$ is
$
    \mathcal{L}_D := P_{x \sim \mathcal{D}}(h(\mathbf{x}) \neq h^*(\mathbf{x}))
$
i.e. the probability of having an example for which $h$ and the teacher disagree. The empirical risk of $h$ is
$
    \mathcal{L}_S(h) := \frac{1}{N}\sum_{n=1}^N \mathbb{I}[h(\mathbf{x}_n) \neq h^*(\mathbf{x}_n)].
$
We consider the hypothesis class of Transformer models, with $TF_\theta$ denoting a Transformer with parameters $\theta$. We consider the usual Transformer definition with softmax attention. We use \textit{width} to denote $\max(d, d_k)$. For more details, refer to Appendix~\ref{subsec:transformers}. We consider fixed precision models, which we make precise through the following:
\begin{definition}[$\mathcal{Q}$-quantized Networks]
Let $\mathcal{Q} \subset \mathbb{R}$ be a finite set such that $0 \in \mathcal{Q}$ and $|\mathcal{Q}| \leq Q$ for some integer $Q \in \mathbb{N}$. A neural network is \emph{$Q$-quantized} if all its parameters are restricted to lie in $\mathcal{Q}$.
\end{definition}
\paragraph{Learning Algorithm}
As in~\cite{chiang2022loss,buzaglo2024uniform}, we assume a Guess and Check (G\&C) learning procedure which operates in the following way:
(i) Sample the weights for a predictor $h$ uniformly at random (ii) Compute $\mathcal{L}_S(h)$; if $\mathcal{L}_S(h)=0$, stop. Note that the rate of success of this learning algorithm is closely tied to the volume of generalizing solutions in the loss landscape, thus connecting our approach to general learning (see \citet{chiang2022loss, buzaglo2024uniform} for details).
\begin{definition}
We say that a predictor $h$ is \textit{teacher-equivalent} (TE) w.r.t. a data distribution $\mathcal{D}$, noted $h \equiv h^\star$, if $\mathbb{P}_{x \sim \mathcal{D}} \left( h(x) = h^\star(x) \right) = 1$.
\end{definition}
Our main theoretical result gives an upper bound on the sample complexity of learning a narrow teacher Transformer implementing an arbitrary C-RASP program:
\begin{theorem}
Fix $\varepsilon\in(0,1)$ and $\delta\in(0,\tfrac15)$. Let $\mathcal{P}$ be a C\text{-}RASP program of depth $n$ with $m$ variables, and suppose there exists a depth-$n$, single--head teacher $TF_{\theta^\star}$ of constant width $2m+c$ (for an absolute constant $c$) implementing $\mathcal{P}$ with $Q$--quantized parameters. Let $TF_\theta$ be any depth-$L \geq n$ single-head student of width $d\ge 2m+c$, also $Q$--quantized, and let $\mathcal{P}_S$ be the posterior obtained by conditioning a uniform prior on zero training error on an i.i.d. sample $S\sim\mathcal{D}^N$. For
\begin{align*}
M_\text{C-RASP} &:= (1 + 2n)(d - 2m - c)
    + n\big((2m + c)^2 + 5(2m + c)d + 2(2m + c)\big) \\
    &\qquad + (L - n)(2m + c)(2d + 1),
\end{align*}
if
\[
N \;\ge\;  \frac{1}{\varepsilon} \big(M_\text{C-RASP}\log Q + 3\log( 2/\delta)\big)
\;\;=\;\; \mathcal{O}\!\left( \frac{1}{\varepsilon} L d \log Q \right)
\]
then
\[
\mathbb{P}_{S\sim\mathcal{D}^N,\;TF_\theta\sim\mathcal{P}_S}\!\left(\mathcal{L}_{\mathcal{D}}(TF_\theta)<\varepsilon\right)\;\ge\;1-\delta.
\]
\end{theorem}

We illustrate the applicability of this theorem through the following example:
\begin{example}[Learning a C-RASP program for Dyck-1]
   Consider the C-RASP program in Algorithm~\ref{alg:dyck1}. This program has $n=7$ operations and $m=9$ variables. Thus, a single-head student with $L \geq 7$ layers and width $d$ could learn this program with
   \begin{align}
       N \geq \mathcal{O} \left (\frac{Ld}{\varepsilon} \right)
   \end{align}
   samples. We note that this gives us a better bound than Occam's razor-like arguments which would yield $\mathcal{O}(Ld^2)$.
\end{example}
\section{Conclusion and Future Work}
In conclusion, we present a simple yet principled framework for deriving sample complexity bounds for Transformer models by leveraging existing results on expressivity. Given a candidate C-RASP program, our analysis characterizes its learnability by a larger model. To the best of our knowledge, this is the first work to study the learnability of C-RASP in Transformers through sample complexity.

Looking forward, several extensions are possible. First, there has been no empirical validation, analogous to~\cite{chiang2022loss}, of how loss landscape impacts length generalization in Transformers/State Space Models.
Moreover, the bound on the probability of sampling a TE model could be tightened, since the current approach only accounts for a single encoding of the teacher within the student.

\bibliography{bibliography}

\begin{thebibliography}{10}
\providecommand{\natexlab}[1]{#1}
\providecommand{\url}[1]{\texttt{#1}}
\expandafter\ifx\csname urlstyle\endcsname\relax
  \providecommand{\doi}[1]{doi: #1}\else
  \providecommand{\doi}{doi: \begingroup \urlstyle{rm}\Url}\fi

\bibitem[Bhattamishra et~al.(2020)Bhattamishra, Ahuja, and
  Goyal]{bhattamishra-etal-2020-ability}
Satwik Bhattamishra, Kabir Ahuja, and Navin Goyal.
\newblock On the {A}bility and {L}imitations of {T}ransformers to {R}ecognize
  {F}ormal {L}anguages.
\newblock In Bonnie Webber, Trevor Cohn, Yulan He, and Yang Liu, editors,
  \emph{Proceedings of the 2020 Conference on Empirical Methods in Natural
  Language Processing (EMNLP)}, pages 7096--7116, Online, November 2020.
  Association for Computational Linguistics.
\newblock \doi{10.18653/v1/2020.emnlp-main.576}.
\newblock URL \url{https://aclanthology.org/2020.emnlp-main.576/}.

\bibitem[Bhattamishra et~al.(2024)Bhattamishra, Hahn, Blunsom, and
  Kanade]{bhattamishra2024separations}
Satwik Bhattamishra, Michael Hahn, Phil Blunsom, and Varun Kanade.
\newblock Separations in the representational capabilities of transformers and
  recurrent architectures.
\newblock \emph{Advances in Neural Information Processing Systems},
  37:\penalty0 36002--36045, 2024.

\bibitem[Buzaglo et~al.(2024)Buzaglo, Harel, Nacson, Brutzkus, Srebro, and
  Soudry]{buzaglo2024uniform}
Gon Buzaglo, Itamar Harel, Mor~Shpigel Nacson, Alon Brutzkus, Nathan Srebro,
  and Daniel Soudry.
\newblock How uniform random weights induce non-uniform bias: Typical
  interpolating neural networks generalize with narrow teachers.
\newblock \emph{arXiv preprint arXiv:2402.06323}, 2024.

\bibitem[Chiang et~al.(2022)Chiang, Ni, Miller, Bansal, Geiping, Goldblum, and
  Goldstein]{chiang2022loss}
Ping-yeh Chiang, Renkun Ni, David~Yu Miller, Arpit Bansal, Jonas Geiping, Micah
  Goldblum, and Tom Goldstein.
\newblock Loss landscapes are all you need: Neural network generalization can
  be explained without the implicit bias of gradient descent.
\newblock In \emph{The Eleventh International Conference on Learning
  Representations}, 2022.

\bibitem[Hahn(2020)]{hahn2020theoretical}
Michael Hahn.
\newblock Theoretical limitations of self-attention in neural sequence models.
\newblock \emph{Transactions of the Association for Computational Linguistics},
  8:\penalty0 156--171, 2020.

\bibitem[Merrill et~al.(2022)Merrill, Sabharwal, and
  Smith]{merrill2022saturated}
William Merrill, Ashish Sabharwal, and Noah~A Smith.
\newblock Saturated transformers are constant-depth threshold circuits.
\newblock \emph{Transactions of the Association for Computational Linguistics},
  10:\penalty0 843--856, 2022.

\bibitem[Rizvi et~al.(2024)Rizvi, Lizaire, Lacroce, and
  Rabusseau]{rizvi2024simulating}
Michael Rizvi, Maude Lizaire, Clara Lacroce, and Guillaume Rabusseau.
\newblock Simulating weighted automata over sequences and trees with
  transformers.
\newblock \emph{arXiv preprint arXiv:2403.09728}, 2024.

\bibitem[Weiss et~al.(2021)Weiss, Goldberg, and Yahav]{weiss2021thinking}
Gail Weiss, Yoav Goldberg, and Eran Yahav.
\newblock Thinking like transformers.
\newblock In \emph{International Conference on Machine Learning}, pages
  11080--11090. PMLR, 2021.

\bibitem[Yang and Chiang(2024)]{yang2024counting}
Andy Yang and David Chiang.
\newblock Counting like transformers: Compiling temporal counting logic into
  softmax transformers.
\newblock \emph{arXiv preprint arXiv:2404.04393}, 2024.

\bibitem[Yang et~al.(2025)Yang, Cadilhac, and Chiang]{yang2025knee}
Andy Yang, Micha{\"e}l Cadilhac, and David Chiang.
\newblock Knee-deep in c-rasp: A transformer depth hierarchy.
\newblock \emph{arXiv preprint arXiv:2506.16055}, 2025.

\end{thebibliography}
\bibliographystyle{plainnat}

\appendix
\section{Preliminaries}
\subsection{Notation}
    We denote with $\mathbb{N}$, $\mathbb{Z}$ and $\R$ the set of natural, integers and real numbers, respectively. We use bold letters for vectors~(\textit{e.g.} $\vec{v} \in \R^{d_1}$), bold uppercase letters for matrices (\textit{e.g.} $\mat{M} \in \R^{d_1 \times d_2}$). All vectors considered are column vectors unless otherwise specified.
    We denote with $\mat{I}$ the identity matrix and  $\mat{0}$ as the matrix full of zeros. The size of such matrices is typically made apparent by the context.
    The $i$-th row and the $j$-th column of a matrix $\mat{M}$ are denoted by $\mat{M}_{i,:}$ and $\mat{M}_{:,j}$.

    Let $\Sigma$ be a fixed finite alphabet of symbols, $\Sigma^*$ the set of all finite strings~(words) with symbols in $\Sigma$ and $\Sigma^n$ the set of all finite strings of length $n$. We use $\varepsilon$ to denote the empty string. Given $p,s \in \Sigma^*$, we denote with $ps$ their concatenation.
\subsection{A More Detailed Introduction to C-RASP}
\label{subsec:crasp}
In this section, we formally define C-RASP and give an overview of C-RASP operations.
\begin{definition}[C-RASP]
A C-RASP program $\mathcal{P}$ is defined as a sequence $P_1, \dots, P_n$ of C-RASP operations. There are two types of operations:
\end{definition}
\begin{minipage}[T]{0.45\textwidth}
\centering
\begin{tabularx}{\linewidth}{l X}
\toprule
\multicolumn{2}{c}{\textbf{Boolean-Valued Operations}} \\
\midrule
\textbf{Initial} & $P(i) := Q_a(i)$ for $a \in \Sigma$ \\
\midrule
\textbf{Boolean} & $P(i) := \neg P_1(i)$ \\
& $P(i) := P_1(i) \land P_2(i)$ \\
\midrule
\textbf{Comparison} & $P(i) := C_1(i) \leq C_2(i)$ \\
\midrule
\textbf{Constant} & $P(i) := 1$ \\
\bottomrule
\end{tabularx}
\end{minipage}
\hfill
\begin{minipage}[T]{0.45\textwidth}
\centering
\begin{tabularx}{\linewidth}{l X}
\toprule
\multicolumn{2}{c}{\textbf{Count-Valued Operations}} \\
\midrule
\textbf{Counting} & $C(i) := \# [j \leq i] \ P(j)$ \\
\midrule
\textbf{Conditional} & $C(i) := P(i) \ ? \ C_1(i) \ : \ C_2(i)$ \\
\midrule
\textbf{Addition} & $C(i) := C_1(i) + C_2(i)$ \\
\midrule
\textbf{Subtraction} & $C(i) := C_1(i) - C_2(i)$ \\
\midrule
\textbf{Min/Max} & $C(i) := \min(C_1(i), C_2(i))$ \\
& $C(i) := \max(C_1(i), C_2(i))$ \\
\midrule
\textbf{Constant} & $C(i) := 1$ \\
\bottomrule
\end{tabularx}
\end{minipage}

Counting operations compute the number of positions $j \leq i$ where $P(j)$ holds, returning their total. By convention, when using a C-RASP program to recognize languages, acceptance is determined by the value of the final operation (required to be Boolean) at the final position. Concretely, if the program runs on input $w$ of length $n$ and the last operation is $D$, then $w$ is accepted if and only if $D(n)$ evaluates to true. Algorithm~\ref{alg:dyck1} gives an example C-RASP program recognizing Dyck-1:
\begin{algorithm}
\caption{C-RASP program recognizing Dyck-1}
\label{alg:dyck1}
\begin{algorithmic}[1]
\Require A string of parentheses $S$.
\State $C_l(i) := \# [j \leq i] Q_l(j)$
\lightgreycomment{The number of '(' up to position $i$}
\State $C_r(i) := \# [j \leq i] Q_r(j)$
\lightgreycomment{The number of ')' up to position $i$}
\State $V(i) := C_r(i) > C_l(i)$
\lightgreycomment{Violation: there are more ')' than '('}
\State $C_v(i) := \# [j \leq i] V(j)$
\lightgreycomment{The number of Violations}
\State $M(i) := C_v(i) = 0$
\lightgreycomment{Matched: zero Violations}
\State $B(i) := C_r(i) = C_l(i)$
\lightgreycomment{Balanced: same number of '(' and ')'}
\State $D(i) := M(i) \land B(i)$
\lightgreycomment{String is Matched and Balanced}
\end{algorithmic}
\end{algorithm}

\paragraph{C-RASP Program Size vs. Transformer Size}
Next, we give an overview of the equivalence results between the size of C-RASP programs and the size of the Transformer models implementing them. We begin with the relationship between model depth and program depth, shown by Theorem 3.1 of~\cite{yang2025knee}:
\begin{theorem}[simplified]
   A language $L$ is defined by a C-RASP program of depth $k$ if and only if it is recognized by a fixed-precision transformer of depth $k$.
\end{theorem}
The original statement makes the equivalence with the underlying temporal logic class, which is itself stated in~\cite{yang2025knee} to have equivalence of depth with C-RASP statements. The following is also true about the width of Transformers implementing C-RASP programs:
\begin{fact}
    Let $|\mathcal{P}|_Q$ represent the number of subformulas, or variables, of a C-RASP program $\mathcal{P}$. A Transformer implementing $\mathcal{P}$ must have width $2|\mathcal{P}|_Q + \mathcal{O}(1)$.
\end{fact}
This fact follows from the construction accompanying Definition 5.6 in~\cite{yang2024counting}.
Note, however, that this is an upper bound. Using other methods, \textit{e.g.}, nearly-orthogonal encoding, it may be possible to do better.

\subsection{Transformer Architecture}
\label{subsec:transformers}
In this section, we detail the architecture of the Transformer we consider for our theoretical results
\begin{definition}[$L$-Layer Transformer]
Let $\mathbf{X} \in \mathbb{R}^{T \times d}$ be an input sequence of length $T$. Define multi-head self-attention with $h$ heads and the MLP sublayer as:
\begin{align}
f_{\text{attn}}(\mathbf{X}) &= \text{Concat}(\mathbf{H}^{(1)}, \dots, \mathbf{H}^{(h)})\mathbf{W}_O + \mathbf{X}, \\
\mathbf{H}^{(i)} &= \text{softmax}\left( \frac{(\mathbf{X}\mathbf{W}_Q^{(i)})(\mathbf{X}\mathbf{W}_K^{(i)})^\top}{\sqrt{d_k}} \right) \mathbf{X}\mathbf{W}_V^{(i)}, \quad i = 1,\dots,h,\\
f_{\text{mlp}}(\mathbf{Z}) &= \phi(\mathbf{Z}\mathbf{W}_1 + \mathbf{b}_1)\mathbf{W}_2 + \mathbf{b}_2 + \mathbf{Z}.
\end{align}
where $\mathbf{W}_Q^{(i)}, \mathbf{W}_K^{(i)}, \mathbf{W}_V^{(i)} \in \mathbb{R}^{d \times d_k}$ for each head $i$, $\mathbf{W}_O \in \mathbb{R}^{hd_k \times d}$ projects the concatenated heads back to model dimension. We assume $d = hd_k$. We assume a 2-layer MLP block with parameters $\mathbf{W}_1 \in \mathbb{R}^{d \times d_{\text{ff}}}$, $\mathbf{W}_2 \in \mathbb{R}^{d_{\text{ff}} \times d}$, $\mathbf{b}_1 \in \mathbb{R}^{d_{\text{ff}}}$, and $\mathbf{b}_2 \in \mathbb{R}^{d}$. Here, $\phi$ is a nonlinearity. In this work, we will assume ReLU nonlinearities unless stated otherwise. This is in line both with many practical implementations of Transformers as well as with the implementation for C-RASP compiled Transformers~\citep{yang2024counting}.
The $L$-layer Transformer is defined recursively as:
\begin{align}
\mathbf{H}^{(0)} &= \mathbf{X}, \\
\mathbf{H}^{(\ell)} &= f_{\text{mlp}}(f_{\text{attn}}(\mathbf{H}^{(\ell-1)})), \quad \text{for } \ell = 1, \dots, L,
\end{align}
We will use $TF_\theta$ to denote a transformer model with parameters $\theta$ and use \textit{width} to denote  $\max(d, d_k)$.
\end{definition}
We assume that at each of the $L$ layers, the number of heads as well as the dimensions of the parameters are the same.
Let $\Theta$ denote the set of all parameters in the model. We denote the parameter count of an $L$-layer $h$-head transformer as
\begin{align*}
M(\Theta) = &1 + d + \sum_{\ell=1}^L
h \left( 3 d  d_k \right)+
d (h  d_k)+
d  d_{\text{ff}}
+
d_{\text{ff}} d +
d_{\text{ff}} +
d.
\end{align*}
For instance, for a 1-layer 1-head Transformer, we would have
\begin{align*}
M(\Theta) = 3dd_k + 2dd_{\text{ff}} + d_{\text{ff}} + 2d + 1.
\end{align*}

\section{Theoretical Results}
\subsection{Previous Results}
We start by recalling the main lemma from~\cite{buzaglo2024uniform}.
\begin{lemma}[G\&C (i.e. Posterior Sampling) Generalization]
\label{lemma:g&c}
Let $\varepsilon \in (0,1)$ and $\delta \in \left(0, \frac{1}{5}\right)$, and assume that $\tilde{p} < \frac{1}{2}$. For any $N$ larger than
\begin{align*}
\frac{-\log(\tilde{p}) + 3 \log\left(\frac{2}{\delta}\right)}{\varepsilon},
\end{align*}
we have that
\begin{align*}
\mathbb{P}_{\mathcal{S} \sim \mathcal{D}^N,\, h \sim \mathcal{P_S}} \left( \mathcal{L}_{\mathcal{D}}(h) < \varepsilon \right) \geq 1 - \delta.
\end{align*}
\end{lemma}
This lemma will be used in conjunction with the following bounds on the probability of sampling a TE model in order to derive the main theorem in this work.
\subsection{Novel Results}
In this section, we restate our main theoretical results and provide proofs for all theoretical statements. The following lemma provides a bound on the probability of sampling a narrow 1-layer Transformer within a larger student Transformer:
\begin{lemma}
\label{lemma:1layersimple}
Let $TF_{\theta^*}$ be a ``narrow'' 1‑layer, $h^*$‑head transformer teacher with dimensions
$
d^*,d_k^*,d_{\text{ff}}^*
$
and $Q$‑quantized parameters, and let $TF_{\theta}$ be any 1-layer $Q$‑quantized wider network with
$
d\geq d^*,\;d_k \geq d_k^*,\;d_{\text{ff}}\geq  d_{\text{ff}}^*,\; h \geq h^*.
$
The teacher network can be embedded in the student with
\begin{align}
        M = (d - d^*)+
        h^*(d^*d_k^* + 2d^*d_k)
        + \tilde h d^* d_k
        + dd^*
        + dd_\text{ff}^* + d^*d_\text{ff} + d + d_\text{ff}
\end{align}
constraints.
\end{lemma}
\begin{proof}
    The key idea of the proof, similar to that of~\cite{buzaglo2024uniform}, is to use the weight matrix from the first layer to zero-out the weight values that interfere with the teacher's activations. This thus gives us activations $\avec = [\avec^*,\;\tilde \avec]$, where $\avec^*$ are the activations of the teacher network and $\tilde \avec$ are the remaining activations from the student. We will use the same convention throughout, using the star to denote teacher parameters and tilde to denote the ``left over'' dimensions. We start by setting the embeddings.

    \paragraph{Embeddings.} Let $\Sigma$ be an alphabet with symbols $\sigma \in \Sigma$. Consider $\xvec(\sigma)$ to be the $d$-dimensional student embedding, we set
    \begin{align}
        \xvec(\sigma) = \begin{bmatrix}
            \xvec^*(\sigma) & \mathbf{0}
        \end{bmatrix},
    \end{align}
    With a slight abuse of notation, we will typically refer to the embedding vector simply as $\xvec$ throughout as the nature of the symbol is not relevant to the key argument of the proof. This gives us $d - d^*$ constraints on the embeddings.

    \paragraph{Attention mechanism} For the first $h^*$ heads (those corresponding to the teacher), we set
    \begin{align}
        \Wmat_Q &=
        \begin{bmatrix}
            \Wmat_{Q}^* & \mathbf{0}\\
            \Wmat_Q^{(21)} & \Wmat_Q^{(22)}
        \end{bmatrix}\\
        \Wmat_K &=
        \begin{bmatrix}
            \Wmat_{K}^*&\Wmat_{K}^{(12)} \\
            \Wmat_K^{(21)}& \Wmat_K^{(22)}
        \end{bmatrix}\\
        \Wmat_V &=
        \begin{bmatrix}
            \Wmat_{V}^* & \mathbf{0}\\
            \Wmat_V^{(21)} & \Wmat_V^{(22)}
        \end{bmatrix}\\
    \end{align}
    Thus for some embedding vector $\xvec$, the computation of the attention weights would give
    \begin{align}
        \begin{bmatrix}
            \xvec^* & \mathbf{0}
        \end{bmatrix}
        \begin{bmatrix}
            \Wmat_{Q}^* & \mathbf{0}\\
            \Wmat_Q^{(21)} & \Wmat_Q^{(22)}
        \end{bmatrix}
        &=
        \begin{bmatrix}
            \qvec^{*\top} & \mathbf{0}
        \end{bmatrix}\\
        \begin{bmatrix}
            \xvec^* & \mathbf{0}
        \end{bmatrix}
        \begin{bmatrix}
            \Wmat_{K}^*&\Wmat_{K}^{(12)} \\
            \Wmat_K^{(21)}& \Wmat_K^{(22)}
        \end{bmatrix}
        &=
        \begin{bmatrix}
            \kvec^{*\top} & \tilde\kvec^\top
        \end{bmatrix},
        \intertext{where $\tilde\kvec$ is the result of the computation with the two rightmost matrices. Computing dot product attention, we obtain}
        \begin{bmatrix}
            \qvec^{*\top} & \mathbf{0}
        \end{bmatrix}
        \begin{bmatrix}
            \kvec^{*} \\
            \tilde\kvec
        \end{bmatrix}
        &= \qvec^{*\top}\kvec^{*}
    \end{align}
    which is the attention computation of the student network. Since attention weights are scalars, we must zero out the remaining activations for either the keys or the queries. For the value matrix, we have
    \begin{align}
        \begin{bmatrix}
            \xvec^* & \mathbf{0}
        \end{bmatrix}
        \begin{bmatrix}
            \Wmat_{V}^* & \mathbf{0}\\
            \Wmat_V^{(21)} & \Wmat_V^{(22)}
        \end{bmatrix}
        &=
        \begin{bmatrix}
            \vvec^{*\top} & \mathbf{0}
        \end{bmatrix}.
    \end{align}
    Letting $\avec$ represent the output of some attention head, we would have
    $\avec = \begin{bmatrix}
            \avec^* & \mathbf{0}
        \end{bmatrix}$
    for the teacher heads. This gives us a total amount of constrained parameters
    $h^*(d^*d_k^* + 2d^*d_k)$.

    For the $h - h^*$ remaining heads, we set no constraints on $\Wmat_Q$ and $\Wmat_K$. We only set constraints on $\Wmat_V$ as follows:
    \begin{align}
    \Wmat_V &=
        \begin{bmatrix}
            \mathbf{0} & \mathbf{0}\\
            \Wmat_V^{(21)} & \Wmat_V^{(22)}
        \end{bmatrix}.
    \end{align}
    This gives us
    \begin{align}
        \begin{bmatrix}
            \xvec^* & \mathbf{0}
        \end{bmatrix}
        \begin{bmatrix}
            \mathbf{0} & \mathbf{0}\\
            \Wmat_V^{(21)} & \Wmat_V^{(22)}
        \end{bmatrix}
        &=
        \mathbf{0}.
    \end{align}
    In essence the key idea is to zero out the outputs of the heads which do not contribute to the teacher computation. This gives us $(h-h^*)d^*d_k$ constraints, or letting $\tilde h = h-h^*$, $\tilde hd^*d_k$ constraints.

    \paragraph{Projection layer}
    Letting $\hvec$ denote the concatenation of the output of all heads, we would have
    \begin{align}
        \hvec = \begin{bmatrix}
            \avec^*_1 & \mathbf{0} &\hdots & \avec^*_{h^*} & \mathbf{0} & \mathbf{0}
        \end{bmatrix},
    \end{align}
    Where we take $\mathbf{0}$ to be the zero vector of the appropriate size given the context. The first $h^*$ subvectors correspond to the outputs of the teacher heads. The last 0 subvector corresponds to the output of the unused student heads.
    For the projection layer, we set
    \begin{align}
        \Wmat^O &=
        \begin{bmatrix}
            \Wmat_{O,1}^* & \Wmat_{O,1}^{(12)} & \Wmat_{O,2}^* & \Wmat_{O,2}^{(12)} & \dots & \Wmat_{O,h^*}^* & \Wmat_{O,h^*}^{(12)} &\Wmat^{(1)}_{O,h^*+1} & \dots & \Wmat^{(1)}_{O,h}\\
            \mathbf{0} & \Wmat_{O,1}^{(22)} & \mathbf{0} & \Wmat_{O,2}^{(22)} & \dots & \mathbf{0} & \Wmat_{O,h^*}^{(22)} & \Wmat^{(2)}_{O,h^*+1} & \dots & \Wmat^{(2)}_{O,h}
        \end{bmatrix}^\top,
    \end{align}
    This gives us an output of the form
    $\begin{bmatrix}
        \ovec^* & \mathbf{0}
    \end{bmatrix}$.
    After applying the skip connections, we obtain
    $\begin{bmatrix}
        \xvec^* + \ovec^* & \mathbf{0}
    \end{bmatrix}$, where the upper dimensions correspond exactly to the residual stream of the teacher. This sets $h^*dd_k^* = dd^*$ constraints on the output projection.

    \paragraph{MLP} Finally, for the MLP, we set
    \begin{align}
        \Wmat_1 &=
        \begin{bmatrix}
            \Wmat_{1}^*&\Wmat_{1}^{(12)} \\
             0 & \Wmat_1^{(22)}
        \end{bmatrix}\\
        \intertext{as well as}
        \bvec_1 &=
        \begin{bmatrix}
            \bvec_1^*\\
            \mathbf{0}
        \end{bmatrix}
        \quad\text{and}\quad
        \Wmat_2 &=
        \begin{bmatrix}
            \Wmat_{2}^*&\Wmat_{2}^{(12)} \\
             0 & \Wmat_2^{(22)}
        \end{bmatrix}\\
        \intertext{as well as}
        \bvec_2 &=
        \begin{bmatrix}
            \bvec_2^*\\
            \mathbf{0}
        \end{bmatrix}
    \end{align}
    Letting $\ovec^\top = [\ovec^*, \mathbf{0}]$ be the output of the attention layer, we can verify this computation for the first layer.
    \begin{align}
        \begin{bmatrix}
            \Wmat_{1}^*&\Wmat_{1}^{(12)} \\
             \mathbf{0} & \Wmat_1^{(22)}
        \end{bmatrix}
        \begin{bmatrix}
            \ovec^*\\
            \mathbf{0}
        \end{bmatrix}
        +
        \begin{bmatrix}
            \bvec_1^* \\
            \mathbf{0}
        \end{bmatrix}
        &=
        \begin{bmatrix}
            \Wmat_1^*\ovec^* + \bvec_1^*\\
            \mathbf{0}
        \end{bmatrix}
    \end{align}
    As the nonlinearity in the MLP is applied position-wise, it does not affect the separation between teacher dimensions and unused student dimensions.
    A similar argument can be made for the second layer as well. Letting $\avec$ represent the activations at the output of the MLP, we obtain after both layers and the skip connections
    $\begin{bmatrix}
        \avec^* + \ovec^*& \mathbf{0}
    \end{bmatrix}$, where the upper dimensions exactly correspond to the residual stream of the teacher.

    Thus, the final count of constrained parameters is
    \begin{align*}
        M =
        (d - d^*) + h^*(d^*d_k^* + 2d^*d_k)
        + \tilde h d^* d_k
        + dd^*
        + dd_\text{ff}^* + d^*d_\text{ff} + d + d_\text{ff},
    \end{align*}
    which concludes the proof.
\end{proof}
\begin{theorem}
\label{thm:Llayersimple}
Let $TF_{\theta^*}$ be a ``narrow'' $L^*$‑layer, $h^*$‑head transformer teacher with dimensions
$
d^*,d_k^*,d_{\text{ff}}^*
$
and $Q$‑quantized parameters, and let $TF_{\theta}$ be any $L$-layer $Q$‑quantized wider network with
$
d\geq d^*,\;d_k \geq d_k^*,\;d_{\text{ff}}\geq  d_{\text{ff}}^*,\; h \geq h^*,\; L\geq L^*.
$
The probability of randomly sampling a TE model satisfies
\begin{align}
\tilde p &\ge Q^{-M},
\end{align}
where
\begin{align*}
M =
(d - d^*)
&+ L^*\big(
        h^*(d^*d_k^* + 2d^*d_k)
        + \tilde h d^* d_k
        + dd^*
        + dd_\text{ff}^* + d^*d_\text{ff} \\
        &\qquad\qquad + d + d_\text{ff}
\big) \\
&+ (L - L^*)\big(hd_k + d_\text{ff} + 1\big)d^*.
\end{align*}
\end{theorem}
\begin{proof}
    The construction proceeds in two parts. We first use the $L^*$ initial layers to implement the teacher network, via repeated application of Lemma~\ref{lemma:1layersimple}. We then use the remaining $L - L^*$ layers to preserve the teacher block in the top $d^*$ coordinates of the residual stream, allowing the bottom $d - d^*$ coordinates to take arbitrary ``garbage'' values $\tilde\rvec$. Intuitively, the only requirement at the end of the computation is that the teacher activations be cleanly separated from the unused dimensions of the student; the bottom coordinates do not need to be preserved across these last layers.

    \paragraph{Teacher layers} We proceed by induction over the first $L^*$ layers.

    \emph{Base case.} For $\ell = 1$, the claim follows directly from Lemma~\ref{lemma:1layersimple}: the input embedding takes the form $[\xvec^*, \mathbf{0}]$, and the layer output is $[\rvec^{*(1)}, \mathbf{0}]$.

    \emph{Induction step.} Assume that the input to attention layer $\ell+1$ is
    \begin{align}
    \rvec^{(\ell)} = \begin{bmatrix}
        \rvec^{*(\ell)} &  \mathbf{0}
    \end{bmatrix}.
    \end{align}
    By Lemma~\ref{lemma:1layersimple}, if the input has the form $[\rvec^{*(\ell)}, \mathbf{0}]$, then so does the output, giving us
    $\rvec^{(\ell+1)} = [\rvec^{*(\ell+1)}, \mathbf{0}]$. After $L^*$ such layers, the residual stream is $[\rvec^{*(L^*)}, \mathbf{0}]$, where $\rvec^{*(L^*)}$ is the output of the teacher network.

    \paragraph{Weak identity layers} The remaining $L - L^*$ layers need only preserve the top $d^*$ coordinates of the residual stream; the bottom $d - d^*$ coordinates are free to take arbitrary values. More precisely, we maintain the invariant
    \begin{align}
        \rvec = \begin{bmatrix}
            \rvec^* & \tilde \rvec
        \end{bmatrix},
    \end{align}
    where $\rvec^*$ is the teacher's output and $\tilde\rvec$ is unconstrained.

    For the attention block, we leave $\Wmat_Q, \Wmat_K, \Wmat_V$ entirely free and set the first $d^*$ columns of $\Wmat_O$ to zero:
    \begin{align}
        \Wmat_O = \begin{bmatrix}
            \mathbf{0} & \tilde\Wmat_O
        \end{bmatrix},
    \end{align}
    where the left block has shape $hd_k \times d^*$ and the right block $\tilde\Wmat_O$ has shape $hd_k \times (d - d^*)$ and is free. Letting $\hvec = \text{Concat}(\mathbf{H}^{(1)}, \dots, \mathbf{H}^{(h)})$, the attention output becomes
    \begin{align}
        \hvec \Wmat_O = \begin{bmatrix}
            \mathbf{0} & \hvec \tilde\Wmat_O
        \end{bmatrix},
    \end{align}
    so the top $d^*$ entries are identically zero regardless of $\hvec$. After the skip connection, the residual stream is
    \begin{align}
        \begin{bmatrix}
            \rvec^* & \tilde\rvec
        \end{bmatrix}
        +
        \begin{bmatrix}
            \mathbf{0} & \hvec \tilde\Wmat_O
        \end{bmatrix}
        =
        \begin{bmatrix}
            \rvec^* & \tilde\rvec + \hvec \tilde\Wmat_O
        \end{bmatrix},
    \end{align}
    preserving the teacher block. This fixes $hd_k d^*$ parameters per layer.

    For the MLP block, we leave $\Wmat_1, \bvec_1$ entirely free and set the first $d^*$ columns of $\Wmat_2$ and the first $d^*$ entries of $\bvec_2$ to zero:
    \begin{align}
        \Wmat_2 = \begin{bmatrix}
            \mathbf{0} & \tilde\Wmat_2
        \end{bmatrix},
        \quad
        \bvec_2 = \begin{bmatrix}
            \mathbf{0} & \tilde\bvec_2
        \end{bmatrix},
    \end{align}
    where the left blocks have shapes $d_\text{ff} \times d^*$ and $d^*$ respectively, and the right blocks are free. Letting $\zvec = \phi(\mat{Z}\Wmat_1 + \bvec_1)$, the MLP output becomes
    \begin{align}
        \zvec \Wmat_2 + \bvec_2 = \begin{bmatrix}
            \mathbf{0} & \zvec \tilde\Wmat_2 + \tilde\bvec_2
        \end{bmatrix},
    \end{align}
    again with the top $d^*$ entries identically zero. After the skip connection, the teacher block is preserved as before. This fixes $(d_\text{ff} + 1) d^*$ parameters per layer.

    Note that the bottom $d - d^*$ coordinates may accumulate arbitrary values across these layers; this does not affect the teacher block, as the zeroing argument above does not depend on the contents of the input. The total number of constraints per weak identity layer is therefore $(hd_k + d_\text{ff} + 1) d^*$.

    \paragraph{Total constraint count} Combining the embedding constraints, the $L^*$ teacher layers (Lemma~\ref{lemma:1layersimple}), and the $L - L^*$ weak identity layers, we obtain
    \begin{align*}
M =
(d - d^*)
&+ L^*\big(
        h^*(d^*d_k^* + 2d^*d_k)
        + \tilde h d^* d_k
        + dd^*
        + dd_\text{ff}^* + d^*d_\text{ff} \\
        &\qquad\qquad + d + d_\text{ff}
\big) \\
&+ (L - L^*)\big(hd_k + d_\text{ff} + 1\big)d^*.
    \end{align*}
    \paragraph{Computation of the bound}
    Let $\mathcal{E}$ be the event corresponding to sampling parameters of the $L$-layer $h$-head Transformer described above. We directly have that
    $\Prob(\mathcal{E}) \leq \Prob(TF_{\theta^*} \equiv TF_\theta)$. Assuming a uniform prior $\mathcal{P}$ over parameters, we get
    \begin{align*}
        \tilde p &\geq \Prob(\mathcal{E}) = Q^{-M},
    \end{align*}
    concluding the proof.
\end{proof}
Combining Theorem~\ref{thm:Llayersimple} with Lemma~\ref{lemma:g&c} immediately yields a sample complexity bound for embedding a narrow teacher into any wider and deeper student.
\begin{corollary}[General sample complexity]
\label{corollary:generalSC}
Fix $\varepsilon \in (0, 1)$ and $\delta \in (0, \tfrac{1}{5})$. Let $TF_{\theta^*}$ be a narrow $L^*$-layer, $h^*$-head teacher with dimensions $d^*, d_k^*, d_\text{ff}^*$ and $Q$-quantized parameters, and let $TF_\theta$ be any $L$-layer, $h$-head student with dimensions
$
d \geq d^*,\; d_k \geq d_k^*,\; d_\text{ff} \geq d_\text{ff}^*,\; h \geq h^*,\; L \geq L^*
$
and $Q$-quantized parameters. Let $\mathcal{P}_S$ be the posterior obtained by conditioning a uniform prior on zero training error on an i.i.d. sample $S \sim \mathcal{D}^N$. If
\[
N \geq \frac{1}{\varepsilon}\big(M \log Q + 3 \log(2/\delta)\big),
\]
with $M$ as defined in Theorem~\ref{thm:Llayersimple}, then
\[
\mathbb{P}_{S \sim \mathcal{D}^N,\;TF_\theta \sim \mathcal{P}_S}\!\left(\mathcal{L}_{\mathcal{D}}(TF_\theta) < \varepsilon\right) \geq 1 - \delta.
\]
\end{corollary}
\begin{proof}
By Theorem~\ref{thm:Llayersimple}, $\tilde p \geq Q^{-M}$, so $-\log \tilde p \leq M \log Q$. Substituting into Lemma~\ref{lemma:g&c} yields the stated sample complexity.
\end{proof}
Next, we relate this bound to the size of a Transformer implementing an arbitrary C-RASP program.
\begin{corollary}
    \label{corollary:crasp-tf}
    Let $\mathcal{P}$ be a C-RASP program of depth $n$ with $m$ subformulas, and consider a single-head student ($h = 1$) of depth $L \geq n$ with $d_k = d_\text{ff} = d$, together with a single-head teacher with $d^*_k = d^*_\text{ff} = d^*$. The probability of sampling a Transformer $TF_{\theta^*}$ implementing $\mathcal{P}$ is bounded by
    \begin{align}
        \tilde p \geq Q^{-M_\text{C-RASP}}
    \end{align}
    with
    \begin{align*}
        M_\text{C-RASP} &:= (1 + 2n)(d - 2m - c)
    + n\Big((2m + c)^2 + 5(2m + c)d + 2(2m + c)\Big) \\
    &\qquad + (L - n)(2m + c)(2d + 1),
    \end{align*}
    where $c$ is an absolute constant.
\end{corollary}
\begin{proof}
    The proof for this statement directly follows from the analysis of the size required to implement a C-RASP program $\mathcal{P}$ given in Section~\ref{subsec:crasp}. One can easily show this fact by simplifying the value of $M$ given in Theorem~\ref{thm:Llayersimple}. We let $h^* = h = 1$, which is in line with~\cite{yang2024counting}, as the constructions they provide for C-RASP operations typically use only one attention head; this also keeps the simplification $d_k = d$ consistent with the standard convention $d = h d_k$ used in Section~\ref{subsec:transformers}.
\end{proof}
We obtain the main theorem for this work by combining Corollary~\ref{corollary:crasp-tf} and Lemma~\ref{lemma:g&c}:
\begin{theorem}
Fix $\varepsilon\in(0,1)$ and $\delta\in(0,\tfrac15)$. Let $\mathcal{P}$ be a C\text{-}RASP program of depth $n$ with $m$ variables, and suppose there exists a depth-$n$, single--head teacher $TF_{\theta^\star}$ of width $2m+c$ (for an absolute constant $c$) implementing $\mathcal{P}$ with $Q$--quantized parameters. Let $TF_\theta$ be any depth-$L \geq n$ single-head student of width $d\ge 2m+c$, also $Q$--quantized, and let $\mathcal{P}_S$ be the posterior obtained by conditioning a uniform prior on zero training error on an i.i.d. sample $S\sim\mathcal{D}^N$. For
\begin{align*}
M_\text{C-RASP} &:= (1 + 2n)(d - 2m - c)
    + n\big((2m + c)^2 + 5(2m + c)d + 2(2m + c)\big) \\
    &\qquad + (L - n)(2m + c)(2d + 1),
\end{align*}
if
\[
N \;\ge\;  \frac{1}{\varepsilon} \big(M_\text{C-RASP}\log Q + 3\log( 2/\delta)\big)
\;\;=\;\; \mathcal{O}\!\left( \frac{1}{\varepsilon} L d \log Q \right)
\]
then
\[
\mathbb{P}_{S\sim\mathcal{D}^N,\;TF_\theta\sim\mathcal{P}_S}\!\left(\mathcal{L}_{\mathcal{D}}(TF_\theta)<\varepsilon\right)\;\ge\;1-\delta.
\]
\end{theorem}

\end{document}